\def\year{2017}\relax
\documentclass[letterpaper]{article}
\usepackage{aaai17}
\usepackage{times}
\usepackage{helvet}
\usepackage{courier}
\usepackage[utf8]{inputenc} 
\usepackage[T1]{fontenc}    
\usepackage{url}            
\usepackage{booktabs}       
\usepackage{amsfonts}       
\usepackage{nicefrac}       
\usepackage{microtype}      

\usepackage{graphicx} 
\usepackage{subfigure}
\usepackage{caption}
\usepackage{enumitem}

\usepackage{amsmath, amsthm, amssymb}

\usepackage{algorithm}
\usepackage{algorithm,algorithmicx}
\usepackage[noend]{algpseudocode}

\newcommand{\citep}[2][]{#1 \cite{#2}}
\newcommand{\citet}[1]{\citeauthor{#1} (\citeyear{#1})}

\include{Definitions}

\newcommand{\Actions}{\mathcal{A}}
\newcommand{\States}{\mathcal{S}}
\newcommand{\Pfcn}{\mathrm{P}}

\newcommand{\Rfcn}{r}

\newcommand{\fa}{x}
\newcommand{\fat}{{\mathbf x}_t}

\newcommand{\Mweight}{m}
\newcommand{\Mweightetd}{m_{\mathrm{ETD}}}
\newcommand{\Dmu}{\mathbf{D}_{\mu}}

\newcommand{\Amat}{\mathbf{A}}
\newcommand{\Bmat}{\mathbf{B}}
\newcommand{\Cmat}{\mathbf{C}}

\newcommand{\Hmat}{\mathbf{H}}

\newcommand{\Qmat}{\mathbf{Q}}

\newcommand{\Umat}{\mathbf{U}}
\newcommand{\Vmat}{\mathbf{V}}

\newcommand{\LL}{\lambda}
\newcommand{\defeq}{\mathrel{\overset{\makebox[0pt]{\mbox{\normalfont\tiny\sffamily def}}}{=}}}
\newcommand{\Lambdamat}{\boldsymbol{\Lambda}}

\newcommand{\Sigmamat}{\boldsymbol{\Sigma}}

\newcommand{\Arank}{\hat{\mathbf{A}}}
\newcommand{\Api}{\mathbf{A}}

\newcommand{\Amatf}[1]{\mathbf{A}_{{\footnotesize #1}}}
\newcommand{\evecf}[1]{\mathbf{e}_{{\footnotesize #1}}}

\newcommand{\bvec}{\mathbf{b}}

\newcommand{\evec}{\mathbf{e}}

\newcommand{\vvec}{\mathbf{v}}
\newcommand{\wvec}{\mathbf{w}}
\newcommand{\xvec}{\mathbf{x}}

\newcommand{\dmu}{{d}_\mu}
\newcommand{\dpi}{{d}_\pi}

\newcommand{\bpi}{\mathbf{b}}

\newcommand{\bvecf}[1]{\mathbf{b}_{{\footnotesize #1}}}

\newcommand{\xdim}{d}
\newcommand{\rdim}{k}
\newcommand{\nsamples}{t}
\newcommand{\eigiter}{j}
\newcommand{\sampiter}{i}

\newcommand{\rankA}{{\text{rank}(\Amat)}}
\newcommand{\picardpower}{p}

\newcommand{\regwgt}{\eta}
\newcommand{\decay}{\beta}
\newcommand{\stepsize}{\alpha}

\newcommand{\inv}{{-1}}
\newcommand{\tpinv}{{-\top}}

\newcommand{\eigmax}{\lambda_{\text{max}}}
\newcommand{\eig}{\lambda}

\newcommand{\E}{\mathbb{E}}

\title{Accelerated Gradient Temporal Difference Learning}
\author{Yangchen Pan, Adam White and Martha White\\
Department of Computer Science\\
Indiana University at Bloomington \\
\texttt{\{yangpan,adamw,martha\}@indiana.edu}}

\begin{document}

\maketitle

\begin{abstract}
%
The family of temporal difference (TD) methods span a spectrum from computationally frugal linear methods like TD($\LL$) to data efficient least squares methods. Least square methods make the best use of available data directly computing the TD solution and thus do not require tuning a typically highly sensitive learning rate parameter, but require quadratic computation and storage. Recent algorithmic developments have yielded several sub-quadratic methods that use an approximation to the least squares TD solution, but incur bias. In this paper, we propose a new family of accelerated gradient TD (ATD) methods that (1) provide similar data efficiency benefits to least-squares methods, at a fraction of the computation and storage (2) significantly reduce parameter sensitivity compared to linear TD methods, and (3) are asymptotically unbiased. We illustrate these claims with a proof of convergence in expectation and experiments on several benchmark domains and a large-scale industrial energy allocation domain.    
\end{abstract}


\section{Introduction}
In reinforcement learning, a common strategy to learn an optimal policy is to iteratively estimate the value function for the current decision making policy---called {\em policy evaluation}---and then update the policy using the estimated values. The overall efficiency of this policy iteration scheme is directly influenced by the efficiency of the policy evaluation step. Temporal difference learning methods perform policy evaluation: they estimate the value function directly from the sequence of states, actions, and rewards produced by an agent interacting with an unknown environment.

The family of temporal difference methods span a spectrum from computationally-frugal, linear, stochastic approximation methods to data efficient but quadratic least squares TD methods. Stochastic approximation methods, such as temporal difference (TD) learning \citep{sutton1988learning} and gradient TD methods \citep{maei2011gradient} perform approximate gradient descent on the mean squared projected Bellman error (MSPBE). These methods require linear (in the number of features) computation per time step and linear memory. These linear TD-based algorithms are well suited to problems with high dimensional feature vectors ---compared to available resources--- and domains where agent interaction occurs at a high rate \citep{szepesvari2010algorithms}. When the amount of data is limited or difficult to acquire, the feature vectors are small, or data efficiency is of primary concern, quadratic least squares TD (LSTD) methods may be preferred. These methods directly compute the value function that minimizes the MSPBE, and thus LSTD computes the same value function to which linear TD methods converge. 
Of course, there are many domains for which neither light weight linear TD methods, nor data efficient least squares methods may be a good match. 

Significant effort has focused on reducing the computation and storage costs of least squares TD methods in order to span the gap between TD and LSTD. The iLSTD method \citep{geramifard2006incremental} achieves sub-quadratic  computation per time step,
but still requires memory that is quadratic in the size of the features.
The tLSTD method \citep{gehring2016incremental} uses an incremental singular value decomposition (SVD) to achieve both sub-quadratic computation and storage. The basic idea is that in many domains the update matrix in LSTD can be replaced with a low rank approximation. In practice tLSTD achieves runtimes much closer to TD compared to iLSTD, 
while achieving better data efficiency. A related idea is to use random projections to reduce computation and storage of LSTD \citep{ghavamzadeh2010lstd}.
In all these approaches, a scalar parameter (descent dimensions, rank, and number of projections), controls the balance between computation cost and quality of solution. 

In this paper we explore a new approach called Accelerated gradient TD (ATD), that performs quasi-second-order gradient descent on the MSPBE. Our aim is to develop a family of algorithms that can interpolate between linear TD methods and LSTD, without incurring bias.
ATD, when combined with a low-rank approximation,
converges in expectation to the TD fixed-point, with convergence rate dependent on the choice of rank. Unlike previous subquadratic methods, consistency is guaranteed even when the rank is chosen to be one.
%
We demonstrate the performance of ATD versus many linear and subquadratic methods in three domains, indicating
that ATD
(1) can match the data efficiency of LSTD, with significantly less computation and storage,
(2) is unbiased, unlike many of the alternative subquadratic methods,
(3) significantly reduces parameter sensitivity for the step-size, versus linear TD methods, 
and
(4) is significantly less sensitive to the choice of rank parameter than tLSTD,
enabling a smaller rank to be chosen and so providing a more efficient incremental algorithm.
Overall, the results suggest that ATD may be the first practical subquadratic complexity TD method suitable for fully incremental policy evaluation.          

\section{Background and Problem Formulation}
In this paper we focus on the problem of {\em policy evaluation}, or that of learning a value function given a fixed policy. We model the interaction between an agent and its environment as a Markov decision process $(\States, \Actions, \Pfcn, \Rfcn)$,
where 
$\States$ denotes the set of states, 
$\Actions$ denotes the set of actions, and 
$\Pfcn: \States \times \Actions \times \States \rightarrow [0,\infty)$ encodes the one-step state transition dynamics. On each discrete time step $t = 1,2,3,...$, the agent selects an action according to its {\em behavior policy}, $A_t\sim\mu(S_t, \cdot)$,
with $\mu: \States \times \Actions \rightarrow [0, \infty)$ and the environment responds by transitioning into a new state $S_{t+1}$ according to $\Pfcn$, and emits a scalar reward $R_{t+1} \defeq \Rfcn(S_t,A_t,S_{t+1})$. 

The objective under policy evaluation is to estimate the {\em value function}, $v_\pi: \States \rightarrow \mathbb{R}$, as the expected return from each state under some {\em target policy} $\pi: \States\times\Actions \rightarrow [0, \infty)$: 
\begin{equation*}
v_\pi(s) \defeq \mathbb{E}_\pi[G_t | S_t = s], 
\end{equation*} 
where $\mathbb{E}_\pi$ denotes the expectation, defined over the future states encountered while selecting actions according to  $\pi$. The {\it return},  denoted by $G_t \in \mathbb{R}$ is the discounted sum of future rewards given actions are selected according to $\pi$:
\begin{align}
G_t &\defeq R_{t+1} + \gamma_{t+1} R_{t+2} + \gamma_{t+1}\gamma_{t+2} R_{t+3} + ... \\
&= R_{t+1} + \gamma_{t+1} G_{t+1} \nonumber
\end{align} 
where $\gamma_{t+1}\in[0,1]$ is a scalar that depends on $S_t, A_t, S_{t+1}$ and discounts the contribution of future rewards exponentially with time. The generalization to transition-based discounting enables the unification of episodic and continuing tasks \citep{white2016unifying} and so we adopt it here. In the standard continuing case, $\gamma_t = \gamma_c$ for some constant $\gamma_c < 1$ and for a standard episodic setting, $\gamma_t = 1$ until the end of an episode, at 
which point $\gamma_{t+1}=0$, ending the infinite sum in the return.
In the most common {\em on-policy} evaluation setting $\pi=\mu$, otherwise $\pi\ne\mu$ and policy evaluation problem is said to be {\em off-policy}.
 
 In domains where the number of states is too large or the state is continuous, it is not feasible to learn the value of each state separately and we must generalize values between states using function approximation. In the case of linear function approximation the state is represented by fixed length feature vectors $\fa: \States \rightarrow \mathbb{R}^\xdim$, where $\fat \defeq \fa(S_t)$ and the approximation to the value function is formed as a linear combination of a learned weight vector, $\wvec\in\mathbb{R}^d$, and $\fa(S_t)$: 
$v_\pi(S_t) \approx  \wvec^\top \xvec_t$. 
The goal of policy evaluation is to learn $\wvec$ from samples generated while following $\mu$.    

The objective we pursue towards this goal is to minimize the mean-squared projected Bellman error (MSPBE):
%
\begin{equation}\label{eq1}
\!\! \text{MSPBE}(\wvec, \Mweight) 
 \! = \! \left(\bvecf{\Mweight}-\Amatf{\Mweight} \wvec \right)^\top\Cmat^\inv \left(\bvecf{\Mweight}-\Amatf{\Mweight} \wvec \right)
\end{equation}
where $\Mweight: \States \rightarrow [0, \infty)$ is a weighting function,
\begin{align*}
\Amatf{\Mweight} &\defeq \E_\mu[\evecf{\Mweight,t} (\xvec_t-\gamma_{t+1}\xvec_{t+1})^\top]\\
\bvecf{\Mweight} &\defeq \E_\mu[R_{t+1}\evecf{\Mweight,t}] \\
\Cmat &\text{ is any positive definite matrix, typically } \Cmat = \E_\mu[\xvec_t \xvec_t^\top]
\end{align*}
%
with
$\bvecf{\Mweight}-\Amatf{\Mweight}\wvec = \E_\mu[\delta_\nsamples(\wvec)\evec_{\Mweight,\nsamples}]$
for TD-error $\delta_t(\wvec) = R_{t+1} + \gamma_{t+1} \xvec_{t+1}^\top \wvec - \xvec_t^\top \wvec$.
The vector $\evecf{\Mweight,t}$ is called the eligibility trace
%
\begin{align*}
\evecf{\Mweight,\nsamples} &\defeq \rho_t(\gamma_{t+1} \lambda \evecf{\Mweight,\nsamples-1}  + M_t\xvec_t) && \triangleright \ \rho_t \defeq \frac{\pi(s_t,a_t)}{\mu(s_t,a_t)}\\
M_t &\text{ s.t. } \ \E_\mu[ M_t | S_t = s_t] = \Mweight(s) / \dmu(s) &&\text{ if } \dmu(s) \neq 0
.
\end{align*} 
where $\lambda\in[0,1]$ is called the trace-decay parameter
and
$\dmu: \States \rightarrow [0, \infty)$ is the stationary distribution induced by following $\mu$. 
The importance sampling ratio $\rho_t$ reweights samples generated by $\mu$ to give an expectation
over $\pi$
%
\begin{align*}
&\E_\mu[\delta_\nsamples(\wvec)\evec_{\Mweight,\nsamples}] \\
&= \sum_{s \in \States} \dmu(s) \E_\pi[\delta_t(\wvec)(\gamma \lambda \evecf{\Mweight,t-1} + M_t\xvec_t) | S_t = s]
.
\end{align*}
This re-weighting enables $v_\pi$ to be learned from samples generated by $\mu$ (under off-policy sampling).

The most well-studied weighting 
occurs when $M_t = 1$ (i.e., $\Mweight(s) = \dmu(s))$.
In the on-policy setting, with $\mu = \pi$, $\rho_t=1$ for all $t$ and $\Mweight(s) = \dpi(s)$
the $\wvec$ that minimizes the MSPBE is the same as the $\wvec$ found by the on-policy temporal difference learning algorithm called TD($\lambda$). 
More recently, a new {\em emphatic} weighting was introduced with the emphatic TD (ETD) algorithm, which we denote
$\Mweightetd$. This weighting includes 
long-term information about $\pi$ (see \cite[Pg. 16]{sutton2016anemphatic}),
%
\begin{align*}
M_t &= \lambda_t + (1-\lambda_t) F_t &&\triangleright \ F_t = \gamma_t \rho_{t-1} F_{t-1} + 1
.
\end{align*} 
Importantly, the $\Amatf{\Mweightetd}$ matrix induced by the emphatic weighting is positive semi-definite \cite{yu2015onconvergence,sutton2016anemphatic}, 
which we will later use to ensure convergence of our algorithm under both on- and off-policy sampling.
The $\Amatf{\dmu}$ used by TD($\lambda$) is not necessarily positive semi-definite, 
and so TD($\lambda$) can diverge when $\pi\ne\mu$ (off-policy). 

Two common strategies to obtain the minimum $\wvec$ of this objective
are stochastic temporal difference techniques, such as TD($\lambda$) \citep{sutton1988learning},
or directly approximating the linear system and solving for the weights, such as in LSTD($\lambda$) \citep{boyan1999least}.
The first class constitute linear complexity methods, both in computation and storage,
including the family of gradient TD methods \citep{maei2011gradient}, 
true online TD methods \citep{vanseijen2014true,vanhasselt2014off}
and several others (see \cite{dann2014policy,white2016investigating} for a more complete summary).
On the other extreme, with quadratic computation and storage,
one can approximate $\Amatf{\Mweight}$ and $\bvecf{\Mweight}$ incrementally
and solve the system $\Amatf{\Mweight} \wvec = \bvecf{\Mweight}$.
Given a batch of $\nsamples$ samples $\{(S_\sampiter,A_\sampiter,S_{\sampiter+1},R_{\sampiter+1})\}_{\sampiter=1}^\nsamples$, one can estimate
\begin{align*}
\Amatf{\Mweight,\nsamples} &\defeq \frac{1}{\nsamples} \sum_{\sampiter=1}^\nsamples \evecf{\Mweight,\sampiter} (\xvec_\sampiter - \gamma \xvec_{\sampiter+1})^\top\\
\bvecf{\Mweight,\nsamples} &\defeq \frac{1}{\nsamples} \sum_{\sampiter=1}^\nsamples \evecf{\Mweight,\sampiter} R_{\sampiter+1},
\end{align*}
 and then compute solution $\wvec$ such that $\Amatf{\Mweight,\nsamples} \wvec = \bvecf{\Mweight,\nsamples}$.
 Least-squares TD methods are typically implemented incrementally using the Sherman-Morrison formula, requiring
 $\mathcal{O}(d^2)$ storage and computation per step.
 
 Our goal is to develop algorithms that interpolate between these two extremes, which we discuss in the next section.

\section{Algorithm derivation}

To derive the new algorithm, we first take the gradient of the MSPBE (in \ref{eq1}) to get
\begin{equation}\label{eq2}
-\frac{1}{2}\nabla_{\wvec}  \text{MSPBE}(\wvec,\Mweight) = \Amatf{\Mweight}^\top \Cmat^\inv \E_\mu[\delta_\nsamples(\wvec)\evec_{\Mweight,\nsamples}].
\end{equation}
Consider a second order update by computing the Hessian: $\Hmat = \Amatf{\Mweight}^\top \Cmat^\inv \Amatf{\Mweight}^\top$.
For simplicity of notation, let
$\Api = \Amatf{\Mweight}$ and $\bpi = \bvecf{\Mweight}$.
For invertible $\Api$, the second-order update is
\begin{align*}
\wvec_{t+1} &=  \wvec_t - \tfrac{\stepsize_t}{2} \Hmat^\inv \nabla_{\wvec}  \text{MSPBE}(\wvec,\Mweight) \\
&= \wvec_t + \stepsize_t (\Api^\top \Cmat^\inv \Api)^\inv\Api^\top \Cmat^\inv \E_\mu[\delta_\nsamples(\wvec)\evec_{\Mweight,\nsamples}]\\
&= \wvec_t + \stepsize_t \Api^\inv \Cmat \Api^{\tpinv}\Api^\top \Cmat^\inv \E_\mu[\delta_\nsamples(\wvec)\evec_{\Mweight,\nsamples}]\\
&= \wvec_t + \stepsize_t \Api^\inv \E_\mu[\delta_\nsamples(\wvec)\evec_{\Mweight,\nsamples}]
\end{align*}
In fact, for our quadratic loss, the optimal descent direction is
$\Api^\inv \E_\mu[\delta_\nsamples(\wvec)\evec_{\Mweight,\nsamples}]$ with $\stepsize_t = 1$, in the sense that
$\argmin_{\Delta \wvec} \text{loss}(\wvec_t + \Delta \wvec) = \Api^\inv \E_\mu[\delta_\nsamples(\wvec)\evec_{\Mweight,\nsamples}]$.
Computing the Hessian and updating $\wvec$ requires quadratic computation,
and in practice quasi-Newton approaches are used that approximate the Hessian. 
Additionally, there have been recent insights that using approximate Hessians for stochastic
gradient descent can in fact speed convergence \citep{schraudolph2007astochastic,bordes2009sgd,mokhtari2014res}.
These methods maintain an approximation to the Hessian, and sample the gradient.
This Hessian approximation provides curvature information that can significantly
speed convergence, as well as reduce parameter sensitivity to the step-size.

Our objective is to improve on the sample efficiency of linear TD methods, 
while avoiding both quadratic computation and asymptotic bias. 
First, we need an approximation $\Arank$ to $\Api$ that provides useful curvature
information, but that is also sub-quadratic in storage and computation. 
Second, we need to ensure that the approximation, $\Arank$,
does not lead to a biased solution $\wvec$.

We propose to achieve this by approximating only $\Amat^\inv$
and sampling $\E_\mu[\delta_\nsamples(\wvec)\evec_{\Mweight,\nsamples}] = \bvec - \Amat \wvec$ using $\delta_t(\wvec_t) \evec_t$ as
an unbiased sample. 
The proposed accelerated temporal difference learning update---which we call ATD($\LL$)---is
\begin{equation*}
\wvec_{t+1} = \wvec_{t} + (\stepsize_t\Arank^\pinv_t+\regwgt \eye)  \delta_t \evec_t
\end{equation*}
with expected update
\begin{equation}
\wvec_{t+1} 
= \wvec_t +  (\stepsize_t\Arank^\pinv + \eta \eye) \E_\mu[\delta_\nsamples(\wvec)\evec_{\Mweight,\nsamples}]
\label{expected2ndorder}
\end{equation}
with regularization $\regwgt > 0$.
If $\Arank$ is a poor approximation of $\Amat$, or discards key information---as we will do with a low rank approximation---
then updating using only $\bvec - \Arank \wvec$ will result in a biased solution, as is the case for tLSTD \cite[Theorem 1]{gehring2016incremental}. Instead, sampling $\bvec - \Amat \wvec = \E_\mu[\delta_\nsamples(\wvec)\evec_{\Mweight,\nsamples}]$, as we show in Theorem \ref{thm_main}, yields 
an unbiased solution, even with a poor approximation $\Arank$.
The regularization $\regwgt > 0$ is key to ensure this consistency, 
by providing a full rank preconditioner $\stepsize_t\Arank^\pinv_t+\regwgt \eye$. 

Given the general form of ATD($\lambda$), the next question is how to approximate $\Api$.  
Two natural choices are a diagonal approximation and a low-rank approximation.
Storing and using a diagonal approximation would only require linear $O(\xdim)$
time and space. For a low-rank approximation $\Arank$, of rank $\rdim$,
represented with truncated singular value decomposition $\Arank = \Umat_\rdim \Sigmamat_\rdim \Vmat_\rdim^\top$,
the storage requirement is $O(\xdim \rdim)$ and
the required matrix-vector multiplications are only $O(\xdim \rdim)$ because for any vector $\vvec$, 
$\Arank \vvec  = \Umat_\rdim \Sigmamat_\rdim (\Vmat_\rdim^\top \vvec)$, is a sequence
of $O(\xdim \rdim)$ matrix-vector multiplications. Exploratory experiments revealed 
that the low-rank approximation approach significantly outperformed the diagonal approximation. 
In general, however, many other approximations to $\Api$ could be used, which is an important direction for ATD.

We opt for an incremental SVD, that previously proved effective
for incremental estimation in reinforcement learning \citep{gehring2016incremental}.
The total computational complexity of the algorithm is $\mathcal{O}(\xdim\rdim + \rdim^3)$
for the fully incremental update to $\Arank$ and $\mathcal{O}(\xdim\rdim)$
for mini-batch updates of $\rdim$ samples.
Notice that when $\rdim=0$, the algorithm reduces exactly to TD$(\LL$),
where $\regwgt$ is the step-size. On the other extreme, where $\rdim = \xdim$, ATD 
is equivalent to an iterative form of LSTD($\LL$). See the appendix for a further discussion,
and implementation details.

\section{Convergence of ATD($\LL$)}\label{sec_convergence}
As with previous convergence results for temporal difference learning algorithms,
the first key step is to prove that the expected update converges to the TD fixed point. 
Unlike previous proofs of convergence in expectation, we do not require the true $\Amat$ to be full rank.
This generalization is important, because as shown previously, $\Api$ is often low-rank, even
if features are linearly independent \cite{bertsekas2007dynamic,gehring2016incremental}.
Further, ATD should be more effective if $\Api$ is low-rank, and so requiring
a full-rank $\Api$ would limit the typical use-cases for ATD. 
%

To get across the main idea, we first prove convergence of ATD with weightings that give positive semi-definite $\Amatf{\Mweight}$;
a more general proof for other weightings is in the appendix. 
\begin{assumption} 
$\Amat$ is diagonalizable, that is, there exists invertible $\Qmat \in \RR^{\xdim \times \xdim}$ with normalized columns (eigenvectors) and diagonal $\Lambdamat \in \RR^{\xdim \times \xdim}$, $\Lambdamat = \diag(\eig_1, \ldots, \eig_\xdim)$, such that $\Amat = \Qmat \Lambdamat \Qmat^\inv$. 
Assume the ordering $\eig_1 \ge \ldots \ge \eig_\xdim$.
\end{assumption}
\begin{assumption}
$\stepsize \in (0,2)$ and $0 < \regwgt \le \eig_1^\inv \max(2 - \stepsize, \stepsize)$.
\end{assumption}
Finally, we introduce an assumption that is only used to
characterize the convergence rate. This condition has been 
previously used \cite{hansen1990thediscrete,gehring2016incremental} to enforce a level of smoothness on the system. 
%
\begin{assumption}
The linear system defined by $\Amat = \Qmat \Lambdamat \Qmat^\inv$ and $\bvec$
satisfy the \textit{discrete Picard condition}: for some $\picardpower > 1$,
$|(\Qmat^\inv \bvec)_\eigiter | \le \eig_\eigiter^\picardpower$
for all $\eigiter = 1,\ldots, \rankA$.
\end{assumption}
\begin{theorem} \label{thm_main}
Under Assumptions 1 and 2, for any $k \ge 0$, let $\Arank$ be the rank-$\rdim$ approximation $\Arank = \Qmat \Lambdamat_\rdim \Qmat^{-1}$ of $\Amatf{\Mweight}$,
where $\Lambdamat_\rdim \in \RR^{\xdim \times \xdim}$ with
$\Lambdamat_\rdim(\eigiter, \eigiter) = \eig_{\eigiter}$ for $j = 1, \ldots, \rdim$  and zero otherwise. 
If 
$\Mweight = \dmu$ or $\Mweightetd$,
the expected updating rule in \eqref{expected2ndorder} converges to the fixed-point $\wvec^\star = \Amatf{\Mweight}^\pinv \bvecf{\Mweight}$. 

\noindent
Further, if Assumption 3 is satisfied, the convergence rate is
\begin{align*}
\vspace{-0.2cm}
\| \wvec_t - \wvec^\star \| 
\le \max\Big( &\max_{\eigiter\in  \{1, \ldots, \rdim\}} |1-\stepsize - \regwgt \eig_{\eigiter}|^{t} \eig_{\eigiter}^{p-1}, 
\\ &\max_{\eigiter\in  \{\rdim+1, \ldots, \rankA \}} |1 - \regwgt \eig_{\eigiter}|^{t} \eig_{\eigiter}^{p-1}\Big)
\end{align*}
\end{theorem}
\begin{proof} 
We use a general result about stationary iterative methods which is applicable to the case where
$\Api$ is not full rank.
\citep[Theorem 1.1]{shi2011convergence} states that given a singular and consistent linear system $\Amat \wvec = \bvec$
%
%
where $\bvec$ is in the range of $\Amat$, the stationary iteration with $\wvec_0 \in \RR^\xdim$ for $t = 1, 2, \ldots$
\begin{align}
\label{iterequ}
\wvec_t = (\eye - \Bmat \Amat) \wvec_{t-1} + \Bmat \bvec
\end{align}
%
converges to the solution $\wvec = \Amat^\pinv \bvec$ 
if and only if the following three conditions are satisfied. 
\begin{enumerate}[leftmargin=*]
\item[] Condition {\bf I}: the  eigenvalues of $\eye - \Bmat \Amat$ are equal to 1 or have absolute value strictly less than 1. 
\item[] Condition {\bf II}: $\text{rank}(\Bmat \Amat) = \text{rank} [(\Bmat \Amat)^2]$.
\item[] Condition {\bf III}: nullspace$(\Bmat \Amat) = $ nullspace$(\Amat)$.
\end{enumerate}
We verify these conditions to prove the result. 
First, because we use the projected Bellman error, $\bpi$ is in the range of
$\Api$ and the system is consistent:
there exists $\wvec$ s.t. $\Api \wvec = \bpi$. 

To rewrite our updating rule \eqref{expected2ndorder} to be expressible in terms of \eqref{iterequ}, let  
$\Bmat = \stepsize \Arank^{\pinv} + \regwgt \eye$, giving
\begin{align}
\Bmat \Amat &= \stepsize \Arank^{\pinv} \Amat + \regwgt \Amat 
= \stepsize \Qmat \Lambdamat_\rdim^\pinv \Qmat^{-1}\Qmat \Lambdamat \Qmat^{-1} + \regwgt \Qmat \Lambdamat \Qmat^{-1} \nonumber\\
&= \stepsize \Qmat \eye_\rdim \Qmat^{-1} + \regwgt \Qmat \Lambdamat \Qmat^{-1} \nonumber\\
&= \Qmat (\stepsize \eye_\rdim + \regwgt  \Lambdamat) \Qmat^{-1} \label{diagAr}
\end{align}
where $\eye_\rdim$ is a diagonal matrix with the indices $1, \ldots, \rdim$ set to 1, and the rest zero. 
\\\\
\textbf{Proof for condition I}. Using \eqref{diagAr}, 
$\eye - \Bmat \Amat = \Qmat (\eye - \stepsize \eye_\rdim - \regwgt \Lambdamat )\Qmat^\inv$.
To bound the maximum absolute value in the diagonal matrix $\eye - \stepsize \eye_\rdim - \regwgt \Lambdamat$,
we consider  eigenvalue $\eig_j$ in $\Lambdamat$, and address two cases. 
Because $\Amatf{\Mweight}$ is positive semi-definite for the assumed $\Mweight$ \citep{sutton2016anemphatic},
$\eig_{\eigiter} \ge 0$ for all $\eigiter = 1, \ldots, \xdim$. 

Case 1: $\eigiter \le \rdim$.
\vspace{-0.4cm}
\begin{align*}
&|1-\stepsize - \regwgt \eig_{\eigiter} | \hspace{1.3cm} \triangleright \text{ for } 0 < \regwgt < \max\left(\frac{2 - \stepsize}{\eig_1}, \frac{\stepsize}{\eig_1} \right)\\
&<  \max(| 1-\stepsize |, |1 - \stepsize - (2-\stepsize)|, |1 - \stepsize -\stepsize|)\\
&=  \max(| 1-\stepsize |, 1, 1) 
< 1 \hspace{1.5cm}  \triangleright \text{ because } \stepsize \in (0,2)
.
\end{align*}

Case 2: $\eigiter > \rdim$. 
 \hspace{0.5cm} $|1 - \regwgt \eig_{\eigiter} | < 1 \hspace{0.5cm}\text{ if } 0 < \regwgt < 2/ \eig_{\eigiter}$
which is true for $\regwgt = \eig_1^\inv\max(2 - \stepsize, \stepsize)$
for any  $\stepsize \in (0,2)$. 
%
%
%
\\
\textbf{Proof for condition II.} $(\Bmat \Amat)^2$ does not change the number of positive eigenvalues, so the rank is unchanged.
\\
\textbf{Proof for condition III}. To show the nullspaces of $\Bmat\Amat$ and $\Amat$ are equal, 
it is sufficient to prove $\Bmat \Amat \wvec = \zerovec$ if and only if $\Amat \wvec = \zerovec$.
$\Bmat = \Qmat (\stepsize \Lambdamat_\rdim + \regwgt \eye) \Qmat^\inv$,
is invertible because $\regwgt > 0$ and $\eig_{\eigiter} \ge 0$.
For any $\wvec \in \text{nullspace}(\Amat)$, we get $\Bmat \Amat \wvec = \Bmat \zerovec = \zerovec$, and
so $\wvec \in \text{nullspace}(\Bmat \Amat)$.
For any $\wvec \in \text{nullspace}(\Bmat\Amat)$, $\Bmat\Amat \wvec = \zerovec \implies \Amat \wvec = \Bmat^\inv\zerovec =\zerovec$, and so $\wvec \in \text{nullspace}(\Amat)$.
\\\\
\textbf{Convergence rate.}
Assume $\wvec_0 = \zerovec$. On each step, we update with
$\wvec_{t+1} = (\eye - \Bmat\Amat) \wvec_t + \Bmat \bvec = \sum_{\sampiter=0}^{t-1} (\eye - \Bmat \Amat)^\sampiter \Bmat \bvec$. This can be verified inductively, where 
\begin{align*}
\wvec_{t+1} &= (\eye - \Bmat\Amat) \sum_{\sampiter=0}^{\nsamples-2} (\eye - \Bmat \Amat)^\sampiter \Bmat \bvec + (\eye - \Bmat\Amat)^0\Bmat \bvec \\
&=\sum_{\sampiter=0}^{\nsamples-1} (\eye - \Bmat \Amat)^\sampiter \Bmat \bvec.
\end{align*}
\newcommand{\Lamsystem}{\bar{\Lambdamat}}
\newcommand{\lamsystem}{\bar{\lambda}}
For $\Lamsystem = \eye - \stepsize \eye_\rdim - \regwgt \Lambdamat$, because $(\eye - \Bmat \Amat)^\sampiter = 
\Qmat \Lamsystem^\sampiter \Qmat^\inv$,
\begin{align*}
\wvec_t &= \Qmat \left(\sum_{\sampiter=0}^{t-1} \Lamsystem^\sampiter \right) \Qmat^\inv \Qmat (\stepsize\Lambdamat_\rdim^\pinv + \regwgt \eye)\Qmat^\inv \bvec\\
&= \Qmat \left(\sum_{\sampiter=0}^{\nsamples-1} \Lamsystem^\sampiter \right) (\stepsize\Lambdamat_\rdim^\pinv + \regwgt \eye)\Qmat^\inv \bvec
\end{align*}
and because $\wvec_\nsamples \rightarrow \wvec^\star$, 
\begin{align*}
\| \wvec_t &- \wvec^\star \| = \|\Qmat \left(\sum_{\sampiter=0}^\infty \Lamsystem^\sampiter - \sum_{\sampiter=0}^t \Lamsystem^\sampiter \right) (\stepsize\Lambdamat_\rdim^\pinv + \regwgt \eye)\Qmat^\inv \bvec \|\\
&= \|\Qmat \Lamsystem_t (\stepsize\Lambdamat_\rdim^\pinv + \regwgt \eye)\Qmat^\inv \bvec \| \ \ \ \ \  \triangleright \Lamsystem_{t}(\eigiter,\eigiter) \defeq  \frac{\lamsystem_\eigiter^{t}}{1-\lamsystem_\eigiter}\\
&\le \|\Qmat \|  \| \Lamsystem_t (\stepsize\Lambdamat_\rdim^\pinv + \regwgt \eye)\Qmat^\inv \bvec \|
\end{align*}
where $ \|\Qmat \| \le 1$ because $\Qmat$ has normalized columns.

For $\eigiter = 1, \ldots, \rdim$, we have that the magnitude of the values in $\Lamsystem_t (\stepsize\Lambdamat_\rdim^\pinv + \regwgt \eye)$ are
\begin{align*}
\frac{(1-\stepsize - \regwgt \eig_{\eigiter})^{t}}{\stepsize + \regwgt\eig_{\eigiter}}(\stepsize \eig_{\eigiter}^\inv + \eta) 
= \frac{(1-\stepsize - \regwgt \eig_{\eigiter})^{t}}{\eig_{\eigiter}}
.
\end{align*}
For $\eigiter = \rdim, \ldots, \rankA$, we get
$\frac{(1 - \eta \eig_{\eigiter})^{t}}{\eig_{\eigiter}}$.

Under the discrete Picard condition, $|(\Qmat^\inv \bvec)_\eigiter | \le \eig_\eigiter^\picardpower$
and so the denominator $\eig_{\eigiter}$ cancels, giving the desired result.
\end{proof}

\vspace{0.2cm}
This theorem gives insight into the utility of ATD for speeding convergence,
as well as the effect of $\rdim$.
Consider TD($\lambda$), which has positive definite $\Amat$ in on-policy learning \cite[Theorem 2]{sutton1988learning}.
The above theorem guarantees ATD convergences to the TD fixed-point, for any $\rdim$.
For $\rdim = 0$, the expected ATD update is exactly the expected TD update.
Now, we can compare the convergence rate of TD and ATD, using the above convergence rate.

Take for instance the setting $\stepsize = 1$ for ATD, which is common for second-order methods
and let $\picardpower=2$. 
The rate of convergence reduces to the maximum of
$\max_{\eigiter\in  \{1, \ldots, \rdim\}} \regwgt^{t} \eig_{\eigiter}^{t+1}$ and 
$\max_{\eigiter\in  \{\rdim+1, \ldots, \rankA \}} |1 - \regwgt \eig_{\eigiter}|^{t} \eig_{\eigiter}$.
In early learning, the convergence rate for TD is dominated by $|1-\regwgt \eig_1|^t \eig_1$,
because $\eig_\eigiter$ is largest relative to $|1-\regwgt \eig_\eigiter|^t$ for small $t$.
ATD, on the other hand, for a larger $\rdim$, can pick a smaller $\regwgt$ 
and so has a much smaller value for $\eigiter=1$, i.e., $\regwgt^t \eig_1^{t+1}$,
and $|1-\regwgt \eig_\eigiter|^t \eig_\eigiter$ is small because $\eig_\eigiter$ is small for $\eigiter > \rdim$.
As $\rdim$ gets smaller, $|1-\regwgt \eig_{\rdim+1}|^t \eig_{\rdim+1}$ becomes larger,
slowing convergence. For low-rank domains, however, $\rdim$ could be quite
small and the preconditioner could still improve the convergence rate in early learning---potentially significantly outperforming TD.  

ATD is a quasi-second order method, meaning sensitivity to parameters should
be reduced and thus it should be simpler to set the parameters.
The convergence rate provides intuition that, for reasonably chosen $\rdim$,
the regularizer $\regwgt$ should be small---smaller than a typical stepsize for TD. 
Additionally, because ATD is a stochastic update,
not the expected update, we make use of typical conventions from  
stochastic gradient descent to set our parameters. We set $\stepsize_t = \frac{\alpha_0}{t}$, 
as in previous stochastic second-order methods \cite{schraudolph2007astochastic},
where we choose $\alpha_0 = 1$ and
set $\regwgt$ to a small fixed value. 
Our choice for $\regwgt$ represents a small final step-size,
as well as matching the convergence rate intuition.



\paragraph{On the bias of subquadratic methods.}
The ATD($\LL$) update was derived to ensure convergence to the minimum of the MSPBE,
either for the on-policy or off-policy setting.
Our algorithm summarizes past information, in $\Arank$, to improve the convergence rate,
without requiring quadratic computation and storage.
Prior work aspired to the same goal, however, the resultant algorithms are biased.
The iLSTD algorithm can be shown to converge for a specific class of feature selection mechanisms \cite[Theorem 2]{geramifard2007ilstd};
this class, however, does not include the greedy mechanism that is used in iLSTD algorithm to select a descent direction. 
The random projections variant of LSTD \citep{ghavamzadeh2010lstd} can significantly reduce the computational complexity compared with conventional LSTD,
with projections down to size $\rdim$, 
but the reduction comes at a cost of an increase in the approximation error \citep{ghavamzadeh2010lstd}.
Fast LSTD \cite{prashanth2013fast} does randomized TD updates on a batch of data; this algorithm could be run incrementally with O($\xdim \rdim$) by using mini-batches of size $\rdim$. 
Though it has a nice theoretical characterization, this algorithm is restricted to $\lambda = 0$.
Finally, the most related algorithm is tLSTD, which also uses a low-rank approximation to $\Amat$.

In ATD $\Arank_t$ is used very differently, from how $\Arank_t$ is used in tLSTD.
The tLSTD algorithm uses a similar approximation $\Arank_t$ as ATD, but tLSTD uses it to compute a closed form solution
$\wvec_t = \Arank_t^\pinv \bvec_t$, and thus is biased \cite[Theorem 1]{gehring2016incremental}. In fact, the bias grows with decreasing $\rdim$,
proportionally to the magnitude of the $\rdim$th largest singular value of $\Amat$. 
In ATD, the choice of $\rdim$ is decoupled from the fixed point, and 
so can be set to balance learning speed and computation with no fear of asymptotic bias.

\section{Empirical Results}
All the following experiments investigate the on-policy setting, and thus we make use of the standard version of ATD for simplicity. Future work will explore off-policy domains with the emphatic update.
The results presented in this section were generated over 756 thousand individual experiments run on three different domains. 
Due to space constraints detailed descriptions of each domain, error calculation, and all other parameter settings are discussed in detail in the appendix. We included a wide variety of baselines in our experiments, additional related baselines excluded from our study are also discussed in the appendix.

\setlength{\belowcaptionskip}{-13pt}
\setlength{\abovecaptionskip}{0pt}

    \begin{figure}[htp]
    \hspace{-0.93cm}
     \includegraphics[width=0.57\textwidth]{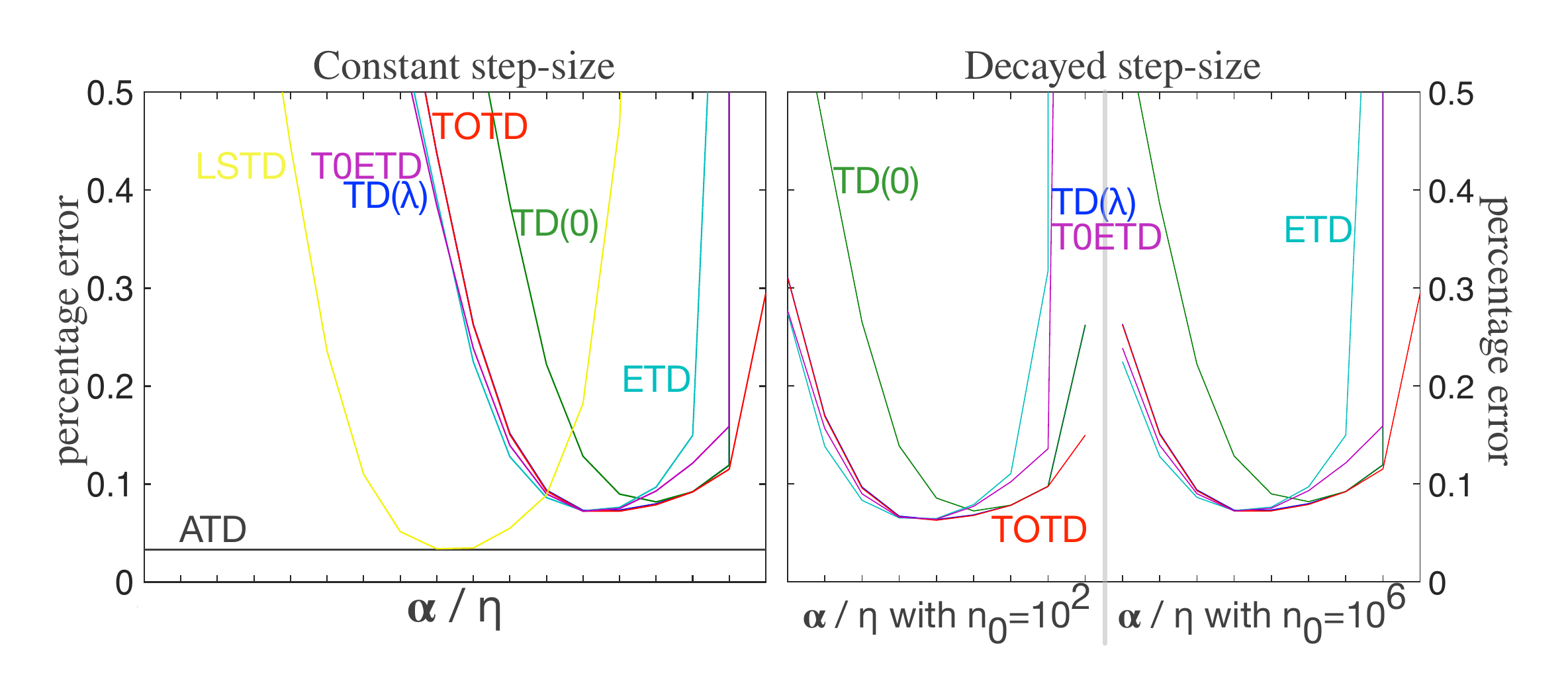}
     \vspace{-0.5cm}
\caption{\label{boyan} \small Parameter sensitivity in Boyan's chain with constant step-size (LHS) and decayed step-sizes (RHS). In the plots above, each point summarizes the mean performance (over 1000 time steps) of an algorithm for one setting of $\alpha$ for linear methods, or $\eta$ for LSTD, and $\alpha/100$ regularizer for ATD, using percentage error compared to the true value function. In the decayed step-size case, where $\alpha_t = \alpha_0\frac{n_0 + 1}{n_0 + \text{episode\#}}$, 18 values of $\alpha_0$ and two values of $n_0$ were tested---corresponding to the two sides of the RHS graph. The LSTD algorithm (in yellow) has no parameters to decay. Our ATD algorithm (in black) achieves the lowest error in this domain, and exhibits little sensitivity to it's regularization parameter (with step-size as $\stepsize_t = \tfrac{1}{t}$ across all experiments).}
\end{figure} 

Our first batch of experiments were conducted on Boyan's chain---a domain known to elicit the strong advantages of LSTD($\LL$) over TD($\LL$). In Boyan's chain the agent's objective is to estimate the value function based on a low-dimensional, dense representation of the underlying state (perfect representation of the value function is possible). The ambition of this experiment was to investigate the performance of ATD in a domain where the pre-conditioner matrix is full rank; no rank truncation is applied. We compared five linear-complexity methods (TD(0), TD($\LL$), true online TD($\LL$), ETD($\LL$), true online ETD($\LL$)), against LSTD($\LL$) and ATD, reporting the percentage error relative to the true value function over the first 1000 steps, averaged over 200 independent runs. We swept a large range of step-size parameters, trace decay rates, and regularization parameters, and tested both fixed and decaying step-size schedules. Figure \ref{boyan} summarizes the results.

Both LSTD($\LL$) and ATD achieve lower error compared to all the linear baselines---even thought each linear method was tuned using 864 combinations of step-sizes and $\LL$. In terms of sensitivity, the choice of step-size for TD(0) and ETD exhibit large effect on performance (indicated by sharp valleys), whereas true-online TD($\LL$) is the least sensitive to learning rate. LSTD($\LL$) using the Sherman-Morrison update (used in many prior empirical studies) is sensitive to the regularization parameter; the parameter free nature of LSTD may be slightly overstated in the literature.\footnote{We are not the first to observe this. \citet{sutton2016reinforcement} note that $\eta$ plays a role similar to the step-size for LSTD.}

    \begin{figure}[htp!]
    \hspace{-0.95cm}
     \includegraphics[width=0.57\textwidth]{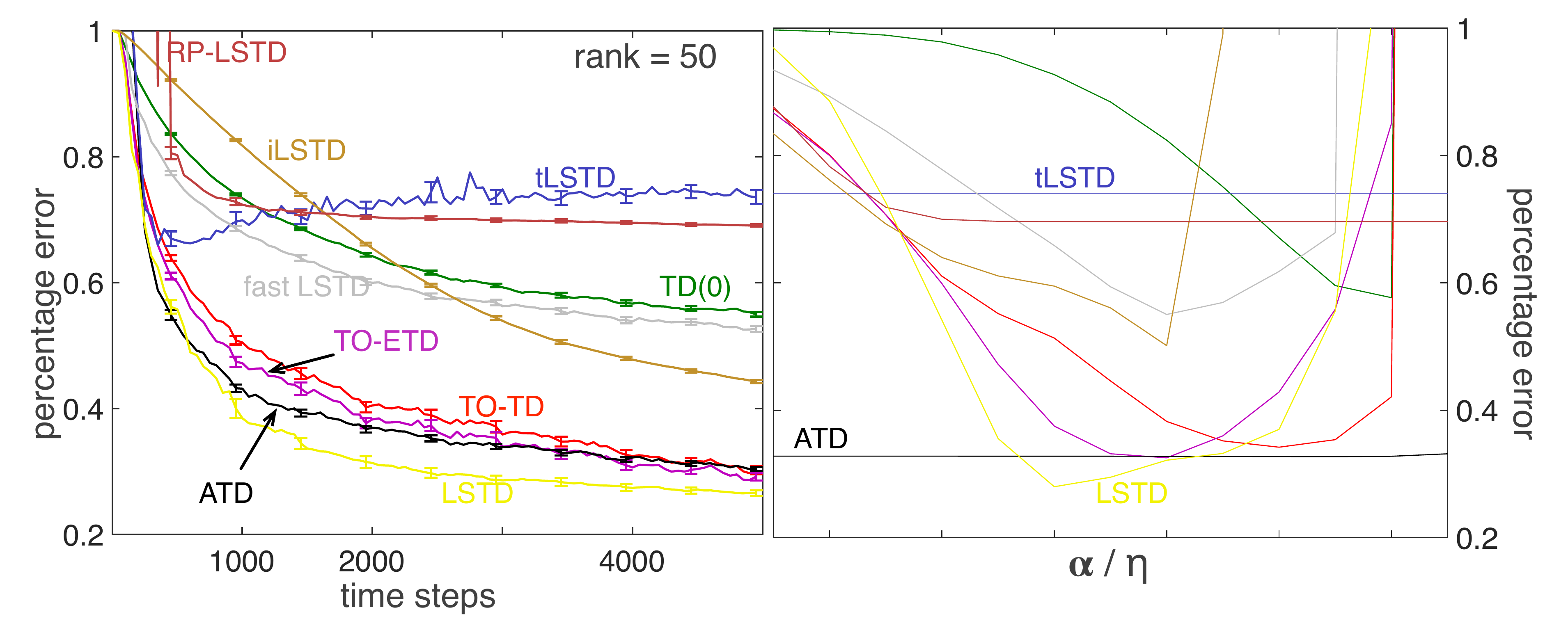}
     \caption{\label{mcar} \small The learning curves (LHS) are percentage error versus time steps averaged over 100 runs of ATD with rank 50, LSTD and several baselines described in text. 
     The sensitivity plot (RHS) is with respect to the learning rate of the linear methods, and regularization parameter of the matrix methods. The tLSTD algorithm has no parameter besides rank, while ATD has little sensitivity to it's regularization parameter.
     }
\end{figure}

Our second batch of experiments investigated characteristics of ATD in a classic benchmark domain with a sparse high-dimensional feature representation where perfect approximation of the value function is not possible---Mountain car with tile coding. The policy to be evaluated stochastically takes the action in the direction of the sign of the velocity, with performance measured by computing a truncated Monte Carlo estimate of the return from states sampled from the stationary distribution (detailed in the appendix). We used a fine grain tile coding of the the 2D state, resulting in a 1024 dimensional feature representation with exactly 10 units active on every time step. We tested TD(0), true online TD($\LL$) true online ETD($\LL$), and sub-quadratic methods, including iLSTD, tLSTD, random projection LSTD, and fast LSTD \cite{prashanth2013fast}. As before a wide range of parameters ($\alpha, \lambda, \eta$) were swept over a large set. Performance was averaged over 100 independent runs. A fixed step-size schedule was used for the linear TD baselines, because that achieved the best performance. The results are summarized in figure \ref{mcar}.

LSTD and ATD exhibit faster initial learning compared to all other methods. This is particularly impressive since $\rdim$ is less than 5\% of the size of $\Amat$. Both fast LSTD and projected LSTD perform considerably worse than the linear TD-methods, while iLSTD exhibits high parameter sensitivity. tLSTD has no tunable parameter besides $\rdim$, but performs  poorly due to the high stochasticity in the policy---additional experiments with randomness in action selection of 0\% and 10\% yielded better performance for tLSTD, but never equal to ATD. The true online linear methods perform very well compared to ATD, but this required sweeping hundreds of combinations of $\alpha$ and $\lambda$, whereas ATD exhibited little sensitivity to it's regularization parameter (see Figure \ref{mcar} RHS); ATD achieved excellent performance with the same parameter setting as we used in Boyan's chain.\footnote{For the remaining experiments in the paper, we excluded the TD methods without true online traces because they perform worse than their true online counterparts in all our experiments. This result matches the results in the literature \citep{vanseijen2016true}.}  

\begin{figure}[htp!]
    \hspace{-0.93cm}
     \includegraphics[width=0.57\textwidth]{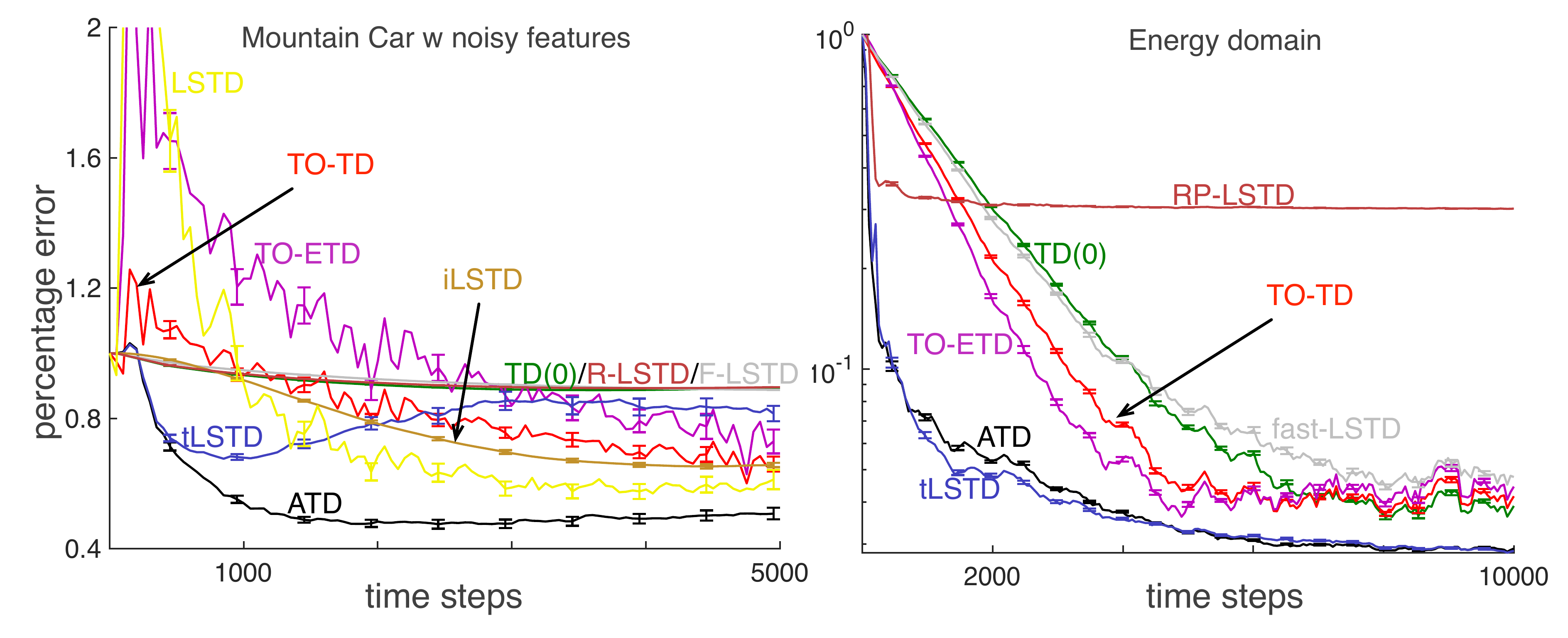}
     \caption{\label{energy} \small Learning curves on Mountain Car with noisy features (LHS) and on Energy allocation (RHS), in logscale.}
\end{figure}

We ran an additional experiment in Mountain car to more clearly exhibit the benefit of ATD over existing methods. We used the same setting as above, except that 100 additional features were added to the feature vector, with 50 of them randomly set to one and the rest zero. This noisy feature vector is meant to emulate a situation such as a robot that has a sensor that becomes unreliable, generating noisy data, but the remaining sensors are still useful for the task at hand. The results are summarized in Figure \ref{energy}. Naturally all methods are adversely effected by this change, however ATD's low rank approximation enables the agent to ignore the unreliable feature information and learn efficiently. tLSTD, as suggested by our previous experiments does not seem to cope well with the increase in stochasticity.  

Our final experiment compares the performance of several sub-quadratic complexity policy evaluation methods in an industrial energy allocation simulator with much larger feature dimension (see Figure \ref{energy}). As before we report percentage error computed from Monte Carlo rollouts, averaging performance over 50 independent runs and selecting and testing parameters from an extensive set (detailed in the appendix). The policy was optimized ahead of time and fixed, and the feature vectors were produced via tile coding, resulting in an 8192 dimensional feature vector with 800 units active on each step. Although the feature dimension here is still relatively small, a quadratic method like LSTD nonetheless would require over 67 million operations per time step, and thus methods that can exploit low rank approximations are of particular interest. The results indicate that both ATD and tLSTD achieve the fastest learning, as expected. The instrinsic rank in this domain appears to be small compared to the feature dimension---which is exploited by ATD and tLSTD with $r = 40$---while the performance of tLSTD indicates that the domain exhibits little stochasticity. The appendix contains additional results for this domain---in the small rank setting ATD significantly outperforms tLSTD.

\section{Conclusion and future work}

In this paper, we introduced a new family of TD learning algorithms
that take a fundamentally different approach from previous incremental TD algorithms.
The key idea is to use a preconditioner on the temporal difference update, similar to a quasi-Newton stochastic gradient descent update. 
We prove that the expected update is consistent, 
and empirically demonstrated improved learning speed and parameter insensitivity, 
even with significant approximations in the preconditioner.

This paper only begins to scratch the surface of potential preconditioners for ATD.  
There remains many avenues to explore the utility of other preconditioners, 
such as diagonal approximations, eigenvalues estimates, other matrix factorizations
and approximations to $\Amat$ that are amenable to inversion. 
The family of ATD algorithms provides a promising avenue for more effectively using results for stochastic gradient descent to improve sample complexity, with feasible computational complexity. 


\bibliographystyle{aaai}
\bibliography{paper}

\newpage

\algnewcommand{\LeftComment}[1]{\State \(\triangleright\) #1}

\appendix

\section{Convergence proof}

%
For the more general setting, where $\Mweight$ can also equal $\Dmu$, we redefine the rank-$\rdim$ approximation.
We say the rank-$\rdim$ approximation $\Arank$ to $\Amat$ is composed of eigenvalues $\{\eig_{i_1}, \ldots, \eig_{i_\rdim}\} \subseteq \{\eig_1, \ldots, \eig_\xdim\}$  if 
$\Arank = \Qmat \Lambdamat_\rdim \Qmat^\inv$ for
diagonal $\Lambdamat_\rdim \in \RR^{\xdim \times \xdim}$,
$\Lambdamat(i_j, i_j) = \eig_{i_j}$ for $j = 1, \ldots, \rdim$,
and zero otherwise. 
\begin{theorem} \label{thm_general}
Under Assumptions 1 and 2, let $\Arank$ be the rank-$\rdim$ approximation composed of eigenvalues $\{\eig_{i_1}, \ldots, \eig_{i_\rdim}\} \subseteq \{\eig_1, \ldots, \eig_\xdim\}$.
If $\eig_\xdim \ge 0$ or $\{\eig_{i_1}, \ldots, \eig_{i_\rdim}\}$ contains all the negative
eigenvalues in $\{\eig_1, \ldots, \eig_\xdim\}$,
then the expected updating rule in \eqref{expected2ndorder} converges to the fixed-point  $\wvec^\star = \Api^\pinv \bpi$.
\end{theorem}
\begin{proof} 
We use a general result about stationary iterative methods \citep{shi2011convergence}, which is applicable to the case where
$\Api$ is not full rank.
\citep[Theorem 1.1]{shi2011convergence} states that given a singular and consistent linear system $\Amat \wvec = \bvec$
%
%
where $\bvec$ is in the range of $\Amat$, the stationary iteration with $\wvec_0 \in \RR^\xdim$ for $t = 1, 2, \ldots$
\begin{align}
\tag{\ref{iterequ}}
\wvec_\sampiter = (\eye - \Bmat \Amat) \wvec_{t-1} + \Bmat \bvec
\end{align}
%
converges to the solution $\wvec = \Amat^\pinv \bvec$ 
if and only if the following three conditions are satisfied. 
\begin{enumerate}[leftmargin=*]
\item[] Condition {\bf I}: the  eigenvalues of $\eye - \Bmat \Amat$ are equal to 1 or have absolute value strictly less than 1. 
\item[] Condition {\bf II}: $\text{rank}(\Bmat \Amat) = \text{rank} [(\Bmat \Amat)^2]$.
\item[] Condition {\bf III}: the null space $\mathcal{N}(\Bmat \Amat) = \mathcal{N} (\Amat)$.
\end{enumerate}
We verify these conditions to prove the result. 
First, because we are using the projected Bellman error, we know that $\bpi$ is in the range of
$\Api$ and the system is consistent:
there exists $\wvec$ s.t. $\Api \wvec = \bpi$. 

To rewrite our updating rule \eqref{expected2ndorder} to be expressible in terms of \eqref{iterequ}, let  
$\Bmat = \stepsize \Arank^{\pinv} + \regwgt \eye$, giving
\begin{align}
\Bmat \Amat &= \stepsize \Arank^{\pinv} \Amat + \regwgt \Amat \nonumber\\
&= \stepsize \Qmat \Lambdamat_\rdim^\pinv \Qmat^{-1}\Qmat \Lambdamat \Qmat^{-1} + \regwgt \Qmat \Lambdamat \Qmat^{-1} \nonumber\\
&= \stepsize \Qmat \eye_\rdim \Qmat^{-1} + \regwgt \Qmat \Lambdamat \Qmat^{-1} \nonumber\\
&= \Qmat (\stepsize \eye_\rdim + \regwgt  \Lambdamat) \Qmat^{-1} \tag{\ref{diagAr}}
\end{align}
where $\eye_\rdim$ is a diagonal matrix with the indices $i_1, \ldots, i_\rdim$ set to 1, and the rest zero. 
\\\\
\textbf{Proof for condition I}. Using \eqref{diagAr}, 
$\eye - \Bmat \Amat = \Qmat (\eye - \stepsize \eye_\rdim - \regwgt \Lambdamat )\Qmat^\inv$.
To bound the maximum absolute value in the diagonal matrix $\eye - \stepsize \eye_\rdim - \regwgt \Lambdamat$,
we consider  eigenvalue $\eig_j$ in $\Lambdamat$, and address three cases. 

Case 1: $j \in \{i_1, \ldots, i_\rdim\}$, $\eig_j \ge 0$:
\begin{align*}
&|1-\stepsize - \regwgt \eig_{\eigiter} | \hspace{1.3cm} \triangleright \text{ for } 0 < \regwgt < \max\left(\frac{2 - \stepsize}{\eig_1}, \frac{\stepsize}{\eig_1} \right)\\ 
&<  \max(| 1-\stepsize |, |1 - \stepsize - (2-\stepsize)|, |1 - \stepsize -\stepsize|)\\
&=  \max(| 1-\stepsize |, 1, 1) 
< 1 \hspace{1.5cm}  \triangleright \text{ because } \stepsize \in (0,2)
.
\end{align*}

Case 2: $j \in \{i_1, \ldots, i_\rdim\}$, $\eig_j < 0$:  $|1 - \stepsize - \regwgt \eig_i | = |1 -\stepsize + \regwgt |\eig_i| | < 1$ if 
$0 \le 1 -\stepsize + \regwgt |\eig_i| < 1 \implies \regwgt < \stepsize/ |\eig_i|$. 

Case 3: $j \notin \{i_1, \ldots, i_\rdim\}$. For this case, $\eig_j \ge 0$, by assumption, as $\{i_1, \ldots, i_\rdim\}$ contains
the indices for all negative eigenvalues of $\Amat$. So $|1 - \regwgt \eig_i | < 1$ if 
$0 < \regwgt < 2/ \eig_i$.

%

All three cases are satisfied by the assumed $\stepsize \in (0,2)$ and $\regwgt \le \eigmax^\inv\max(2 - \stepsize, \stepsize)$.
Therefore, the absolute value of the eigenvalues of $\eye - \Bmat \Amat$ are all less than 1 and so the first condition holds. 
\\\\
\textbf{Proof for condition II.} $(\Bmat \Amat)^2$ does not change the number of positive eigenvalues, so the rank is unchanged.
\begin{align*}
\Bmat \Amat &=  \Qmat (\stepsize \eye_\rdim + \regwgt  \Lambdamat) \Qmat^{-1} \\
(\Bmat \Amat)^2  &= \Qmat (\stepsize \eye_\rdim + \regwgt \Lambdamat )\Qmat^\inv \Qmat (\stepsize \eye_\rdim + \regwgt \Lambdamat )\Qmat^\inv \\
&= \Qmat (\stepsize \eye_\rdim + \regwgt \Lambdamat )^2 \Qmat^\inv
\end{align*}
\\
\textbf{Proof for condition III}. To show that the nullspaces of $\Bmat\Amat$ and $\Amat$ are equal, 
it is sufficient to prove $\Bmat \Amat \wvec = \zerovec$ if and only if $\Amat \wvec = \zerovec$.
Because $\Bmat = \Qmat (\stepsize \Lambdamat_\rdim + \regwgt \eye) \Qmat^\inv$,
we know that $\Bmat$ is invertible as long as $\stepsize \neq -\regwgt\eig_j$. Because $\regwgt > 0$,
this is clearly true for $\eig_j \ge 0$ and also true for $\eig_j < 0$ because $\regwgt$ is strictly 
less than $\stepsize/|\eig_j|$. 
For any $\wvec \in \text{nullspace}(\Amat)$, we get $\Bmat \Amat \wvec = \Bmat \zerovec = \zerovec$, and
so $\wvec \in \text{nullspace}(\Bmat \Amat)$.
For any $\wvec \in \text{nullspace}(\Bmat\Amat)$, we get $\Bmat\Amat \wvec = \zerovec \implies \Amat \wvec = \Bmat^\inv\zerovec =\zerovec$, and so $\wvec \in \text{nullspace}(\Amat)$,
completing the proof.
\end{proof}

With $\rdim = \xdim$, the update is a gradient descent update on the MSPBE,
and so will converge even under off-policy sampling. As $\rdim << \xdim$, 
the gradient is only approximate and theoretical results about (stochastic) gradient
descent no longer obviously apply. For this reason, we use the iterative update analysis above to
understand convergence properties. 
Iterative updates for the full expected update, with preconditioners, have been studied
in reinforcement learning (c.f. \citep{wang2013ontheconvergence});
however, they typically analyzed different preconditioners, as they had no
requirements for reducing computation below quadratic computation.
For example, they consider a regularized preconditioner $\Bmat = (\Amat + \regwgt \eye)^\inv$,
which is not compatible with an incremental singular value decomposition
and, to the best of our knowledge, current iterative eigenvalue decompositions require symmetric matrices. 

The theorem is agnostic to what components of $\Amat$ are approximated
by the rank-$\rdim$ matrix $\Arank$. In general, a natural choice, particularly in on-policy learning
or more generally with a positive definite $\Amat$,
is to select the largest magnitude eigenvalues of $\Amat$, which contain the most significant information about
the system and so are likely to give the most useful curvature information.
However, $\Arank$ could also potentially be chosen to obtain convergence for
off-policy learning with $\Mweight = \dmu$, where $\Amat$ is not necessarily positive semi-definite. 
This theorem indicates that if the rank $\rdim$ approximation 
$\Arank$ contains the negative eigenvalues of $\Amat$, even if it does not
contain the remaining information in $\Amat$, then we obtain convergence under off-policy sampling.
We can of course use the emphatic weighting more easily for off-policy learning,
but if the weighting $\Mweight = \dmu$ is desired rather than $\Mweightetd$,
then carefully selecting $\Arank$ for ATD enables that choice.

\begin{algorithm*}[htp!]
	\caption{Accelerated Temporal Difference Learning}
	\label{alg_atd}
	\begin{algorithmic}
		\LeftComment{where $\Umat_0 = [], \Vmat_0 = [], \Sigmamat_0 = [], \bvec_0 = \vec{\zerovec}, \evec_0 = \vec{\zerovec},$~ initialized $\wvec_0$~arbitrarily} 
		\State
		\Function{ATD}{$\rdim, \regwgt, \LL$}
		\State  $\xvec_0 =$ first observation 
		\State $\regwgt = $ a small final stepsize value, e.g., $\regwgt = 10e-4$ 
		\For {$t=0,1,2,...$} 
		\State In $\xvec_t$, select action $A_t~\sim\pi$, observe $\xvec_{t+1}$, reward $r_{t+1}$, discount $\gamma_{t+1}$ (could be zero if terminal state)
		\State $\decay_t = 1/(t+1)$
		\State $\delta_t = r_{t+1} + \gamma_{t+1}\wvec^\top \xvec_{t+1} - \wvec^\top \xvec_t$
		\State $\evec_t =$~ \Call{trace\_update}{$\evec_{t-1}, \xvec_t, \gamma_t, \lambda_t$} \Comment{or call \Call{empahtic\_trace\_update}{} to use emphatic weighting}
		\LeftComment{$\Umat_{t}\Sigmamat_{t}\Vmat^\top_{t} = (1-\decay_t)\Umat_{t-1}\Sigmamat_{t-1}\Vmat^\top_{t-1} + \decay_t\evec_t (\xvec_t - \gamma_{t+1} \xvec_{t+1})^\top$} 
		\State $[\Umat_{t},\Sigmamat_{t} , \Vmat_{t}] \!=\! $ \Call{svd-update}{$\Umat_{t-1},(1-\decay_t)\Sigmamat_{t-1} ,\Vmat_{t-1},\sqrt{\decay_t}\evec_t, \sqrt{\decay_t}(\xvec_t \!-\! \gamma_{t+1} \xvec_{t+1}),\rdim}$ \Comment{see \cite{gehring2016incremental}}
		\LeftComment{Ordering of matrix operations important, first multiply $\Umat^\top_{t} \evec_t$ in $\mathcal{O}(\xdim \rdim)$ time}
		\LeftComment{to get a new vector, then by $\Sigmamat^\pinv_{t}$ and $\Vmat_{t}$ to maintain only matrix-vector multiplications}
		\State $\wvec_{t+1} = \wvec_{t} +  (\tfrac{1}{t+1}\Vmat_{t} \Sigmamat^\pinv_{t} \Umat^\top_{t} +\regwgt \eye) ( \delta_t \evec_t)$	\Comment{where $\Sigmamat^\pinv_{t} = \diag(\hat{\sigma}_1^\inv, \ldots, \hat{\sigma}_\rdim^\inv, \zerovec)$}	
		\EndFor
		\EndFunction
	\end{algorithmic}
\end{algorithm*}

\section{Algorithmic details}

In this section, we outline the implemented ATD($\LL$) algorithm.
The key choices are how to update the approximation to $\Arank$,
and how to update the eligibility trace to obtain different variants of TD.
We include both the conventional and emphatic trace updates
in Algorithms 2 and 3 respectively. The low-rank update to $\Arank$
uses an incremental singular value decomposition (SVD).
This update to $\Arank$ is the same one used for tLSTD, and
so we refer the reader to  \cite[Algorithm 3]{gehring2016incremental}.
The general idea is to incorporate the rank one update $\evec_t(\xvec_t - \gamma_{t+1} \xvec_{t+1})^\top$
into the current SVD of $\Arank$. In addition, to maintain a normalized $\Arank$,
we multiply by $\beta_t$: 
\begin{align*}
\Arank_{t+1} 
&= (1-\beta_t) \Arank + \beta_t \evec_t(\xvec_t - \gamma_{t+1} \xvec_{t+1})^\top\\
&= \tfrac{t}{t+1} \Arank + \tfrac{1}{t+1} \evec_t(\xvec_t - \gamma_{t+1} \xvec_{t+1})^\top
\end{align*}
Multiplying $\Arank$ by a constant corresponds to multiplying the singular values.
We also find that multiplying each component of the rank one update by $\sqrt{1/(t+1)}$
is more effective than multiplying only one of them by $1/(t+1)$.

\begin{algorithm}[H]
\caption{Conventional trace update for ATD}
\label{conventional_trace}
\begin{algorithmic}
\State
\Function{trace\_update}{$\evec_{t-1}, \xvec_t, \gamma_t, \lambda_t$}
\State \Return $\gamma_t\lambda_t \evec_{t-1} + \xvec_t$
\EndFunction
\end{algorithmic}
\end{algorithm}
\begin{algorithm}[H]
\caption{Emphatic trace update for ATD}
\label{etd_trace}
\begin{algorithmic}
\LeftComment{where  $F_0 \leftarrow 0, \ \ \  M_0 \leftarrow 0$ is initialized globally, before executing the for loop in ATD($\LL$)}
\State
\Function{Emphatic\_trace\_update}{$\evec_{t-1}, \xvec_t, \gamma_t, \lambda_t$}
\State $\rho_t \gets \frac{\pi(s_t,a_t)}{\mu(s_t,a_t)}$ \Comment{Where $\rho_t = 1$ in the on-policy case}
\State $F_t \gets  \rho_{t-1}\gamma_t F_{t-1} + i_t$  \Comment{For uniform interest, $i_t = 1$}
\State $M_t \gets \lambda_t i_t + (1-\lambda_t) F_t$
\State
\State \Return $\rho_t (\gamma_t\lambda_t \evec_{t-1} + M_t\xvec_t)$
\EndFunction
\end{algorithmic}
\end{algorithm}

\section{Detailed experimental specification} 
In both mountain car and energy storage domains we do not have access to the parameter's of the underlying MDPs (as we do in Boyan's chain), and thus must turn to Monte Carlo rollouts to estimate $v_\pi$ in order to evaluate our value function approximation methods. In both domains we followed the same strategy. 

To generate training data we generated 100 trajectories of rewards and observations under the target policy, 
starting in randomly from a small area near a start state. 
Each trajectory is composed of a fixed number of steps, either 5000 or 10000, and, in the case of episodic tasks like mountain car, may contain many episodes. The start states for each trajectory were sampled uniform randomly from (1) near the bottom of the hill with zero velocity for mountain car, (2) a small set of valid start states specified by the energy storage domains \citep{salas2013benchmarking}. Each trajectory represents one independent run of the domain.

The testing data was sampled according to the on-policy distribution induced by the target policy. For both domains we generated a single long trajectory selecting actions according to $\pi$. Then we randomly sampled 2000 states from this one trajectory. In mountain car domain, we ran $500$ Monte Carlo rollouts to compute undiscounted sum of future rewards until termination, and take the average as an estimate true value. In the energy allocation domain, we ran $300$ Monte Carlo rollouts for each evaluation state, each with length $1000$ steps\footnote{After $1000$ steps, for $\gamma = 0.99$, the reward in the return is multiplied by $\gamma^{1000} < 10^{-5}$ and so contributes a negligible amount to the return.}, averaging over $300$ trajectories from each of the evaluation states. We evaluated the algorithms' performances by comparing the agent's prediction value with the estimated value of the $2000$ evaluation states, at every $50$ steps during training. We measured the percentage absolution mean error:

$${\text{error}}(\wvec) = \frac{1}{2000}\sum_{i=1}^{2000}\frac{|\wvec^T \xvec(s_i) - \hat{v}_\pi(s_i)|}{|\hat{v}_\pi(s_i)|} ,$$
where $\hat{v}_\pi(s_i)\in\mathbb{R}$ denotes the Monte Carlo estimate of the value of evaluation state $s_i$.    

\subsection{Algorithms}
The algorithms included in the experiments constitute a wide range of stochastic approximation
algorithms and matrix-based (subquadratic) algorithms. There are a few related algorithms, however,
that we chose not to include; for completeness we explain our decision making here. 

 There have been some accelerations proposed to gradient TD
 algorithms \citep{mahadevan2014proximal,meyer2014accelerated,dabney2014natural}. However, they have either shown
 to perform poorly in practice \citep{white2016investigating}, or were based on applying accelerations outside
 their intended use \citep{meyer2014accelerated,dabney2014natural}.
 \citet{dabney2014natural} explored a similar update to ATD, but for the control setting and with an incremental update to the Fisher information matrix rather than $\Amat$ used here. As they acknowledge, this approach for TD methods is somewhat adhoc, as the typical update is not a gradient, and rather their method
 is better suited for the policy gradient algorithms explored in that paper.
 \citet{meyer2014accelerated} applied an accelerated Nesterov technique, called SAGE, to the two timescale gradient algorithms.
 Their approach does not take advantage of the simpler quadratic form of the MSPBE, and so only uses an approximate Lipschitz constant to improve selection of the stepsize. Diagonal approximations to $\Amat$ constitute a strictly more informative stepsize approach, and we found these to be inferior to our low-rank strategy. The results by \citet{meyer2014accelerated} using SAGE for GTD similarly indicated little to no gain. 
 Finally, \citet{givchi2014quasi} adapted SGD-QN for TD, and showed some improvements using this diagonal step-size approximation. 
  
 On the other hand, the true-online methods have consistently been shown
 to have surprisingly strong performance \citep{white2016investigating},
 and so we opt instead for these practical competitors. 
 
\subsection{Boyan's Chain} 
This domain was implemented exactly as describe in Boyan's paper \citep{boyan1999least}. The task is episodic and the true value function is known, and thus we did not need to compute rollouts. Otherwise evaluation was performed exactly as described above. We tested the following parameter settings:
\begin{itemize} 
\item $\alpha_0 \in \{0.1\times2.0^j | j = -12,-11,-10, ...,4,5\}$, 18 values in total\\
\item $n_0 \in \{10^2,10^6\}$\\
\item $\lambda \in \{0.0, 0.1, ..., 0.9, 0.91,0.93,0.95,0.97,0.99,1.0\}$, 16 values in total\\
\item $\eta \in \{10^j | j = -4,-3.5,-3, ...,3.5,4,4.5\}$, 18 values in total.
\end{itemize}
The linear methods, (e.g., TD(0) true online ETD($\LL$)), made use of $\alpha_0, n_0$, and $\lambda$, whereas the LSTD made use of $\eta$ to initialize the incremental approximation of $\Amat$ inverse and $\lambda$. For the linear methods we also tested decaying step size schedule as originally investigated by Boyan
$$\alpha_t = \alpha_0\frac{n_0 + 1}{n_0 + \text{\#terminations }}.$$
We also tested constant step-sizes where $\alpha_t = \alpha_0$.  
The ATD algorithm, as proposed was tested with one fixed parameter setting.       
 
\subsection{Mountain Car}
Our second batch of experiments was conducted on the classic RL benchmark domain Mountain Car. We used the \citet{sutton1998reinforcement} specification of the domain, where the agent's objective is to select one of three discrete actions (reverse, coast, forward), based on the continuous position and velocity of an underpowered car to drive it out of a valley, at which time the episode terminates. This is an undiscounted task. Each episode begins at the standard initial location --- randomly near the bottom of the hill --- with zero velocity. Actions were selected according to a stochastic {\em Bang-bang} policy, where reverse is selected if the velocity is negative and forward is selected if the velocity is positive and occasionally a random action is selected---we tested randomness in action selection of 0\%, 10\%, and 20\%.

We used tile coding to convert the continuous state variable into high-dimensional binary feature vectors. The position and velocity we tile coded jointly with 10 tilings, each forming a two dimensional uniform grid partitioned by 10 tiles in each dimension. This resulted in a binary feature vector of length 1000, with exactly 10 components equal to one and the remaining equal to zero. We requested 1024 memory size to guarantee the performance of tile coder, resulted in finally 1024 features. We used a standard freely available implementation of tile coding \footnote{https://webdocs.cs.ualberta.ca/~sutton/tiles2.html}, which is described in detail in \citet{sutton1998reinforcement}.

We tested the following parameter settings for Mountain Car:
\begin{itemize} 
\item $\alpha_0 \in \{0.1 \times 2.0^j | j = -7,-6, ...,4,5\}$ divided by number of tilings, 13 values in total\\
\item $\lambda \in \{0.0, 0.1, ..., 0.9, 0.93,0.95,0.97,0.99, 1.0\}$, 15 values in total \\
\item $\eta \in \{10^j | j = -4,-3.25,-2.5, ...,3.5,4.25,5.0\}$, 13 values in total.
\end{itemize}
The linear methods (e.g., TD(0)), iLSTD, and fast LSTD made use of $\alpha_0$ as stepsize, ATD uses $\alpha_0/100$ as regularizer, whereas the LSTD, and random projection LSTD made use of the $\eta$ as regularization for Sherman-Morrison matrix initialization. All methods except fast LSTD and TD(0) made use of the $\lambda$ parameter. iLSTD used decaying step-sizes with $n_0 = 10^2$. In addition we fixed the number of descent dimensions for iLSTD to one (recommended by previous studies \citep{geramifard2006incremental,geramifard2007ilstd}). We found that the linear methods, on the other hand, performed worse in this domain with decayed step-sizes so we only reported the performance for the constant step size setting. In this domain we tested several settings for the regularization parameter for ATD. However, as the results demonstrate ATD is insensitive to this parameter. Therefore we present results with the same fixed parameter setting for ATD as used in Boyan's chain. The low rank matrix methods---including ATD---were tested with rank equal to 20, 30, 40, 50, and 100. With rank 20, 30, 40, we observed that ATD can still do reasonably well but converges slower. However, rank = 100 setting does not show obvious strength, likely due to the fact that the threshold for inversing $\Amat$ remains unchanged.
 
\subsection{Energy Allocation}
Our final experiments were run on a simulator of a complex energy storage and allocation task. This domain simulates control of a storage device that interacts with a market and stochastic source of energy as a continuing discounted RL tasks. The problem was originally modeled as a finite horizon undiscounted task \citep{salas2013benchmarking}, with four state variables at each time step: the amount of energy in the storage device $R_t$, the net amount of wind energy $E_t$, time aggregate demand $D_t$, and price of electricity $P_t$ in spot market. The reward function encodes the revenue earned by the agent's energy allocation strategy as a real value number. The policy to be evaluated was produced by an approximate dynamic programming algorithm from the literature\citep{salas2013benchmarking}. The simulation program is from \emph{Energy storage datasets II} from \emph{http://castlelab.princeton.edu}. 

We made several minor modifications to the simulator to allow generating training or testing data for policy evaluation. First, we modified the original policy by setting the input time index as $ (\text{\#timeindex} \mod 24)$ so that we can remove the restriction that time index must be no greater than $24$. Though no longer an optimal policy, this still constitutes a valid policy that provides the same distribution over actions for a given state. Second, we added an additional variable, $D_{t-1}$, to the state at time $t$, encoding the state as five variables. This addition was to ensure a Markov state, are using only the original four variables results in a time-dependent state. Third, we considered the problem as a continuing task by setting discount rate ($\gamma = 0.99$) when estimating values of states. 

We used the following parameter setting of generating the training data and testing data. The stochastic processes associated with $P_t, E_t, D_t$ are jump process, uniform process and sinusoidal process. The ranges of $R_t, E_t, P_t, D_t$ are: $[0, 30], [1, 7], [30, 70], [0, 7]$. When generating the training trajectories, we randomly choose each state values from the ranges: $[0, 10], [1, 5], [30, 50], [0, 7]$. 

Again, we used tile coding to convert the state variable into high-dimensional binary feature vectors, similar to how the acrobot domain was encoded in prior work (see Sutton \& Barto, 1998). We tile coded all 3-wise combinations, all pair-wise combinations, and each of the five state variables independently (sometimes called stripped tilings). More specifically we used:
\begin{itemize}
\item all five one-wise tilings of 5 state variables, with gridsize = 4, numtilings = 32 (memory = $5\times4\times32$)\\
\item all ten two-wise tilings of 5 state variables, with gridsize = 4, numtilings = 32 (memory = $10\times4^2\times32$)\\
\item all ten three-wise tilings of 5 state variables, with gridsize = 2, numtilings = 32 (memory = $10\times2^3\times32$).\\
\end{itemize}
This resulted in a binary feature vector of length 8320, which we hashed down to $8192 = 2^{13}$.    
Training data and evaluation were conducted in the exact same manner as the Mountain car experiment. We tested a similar set of parameters as before:
\begin{itemize}
\item $\alpha_0 \in \{0.1\times2.0^j | j = -7,-6, ...,4,5\}$ divided by number of tilings, 13 values in total\\
\item $\lambda \in \{0.0, 0.1, ..., 0.9, 1.0\}$, 10 values in total\\
\item $\eta \in \{10^j | j = -4,-3.25,-2.5, ...,3.5,4.25,5.0\}$, 13 values in total.
\end{itemize}

\begin{figure}[htp]
\hspace{-0.5cm}
     \includegraphics[width=0.5\textwidth]{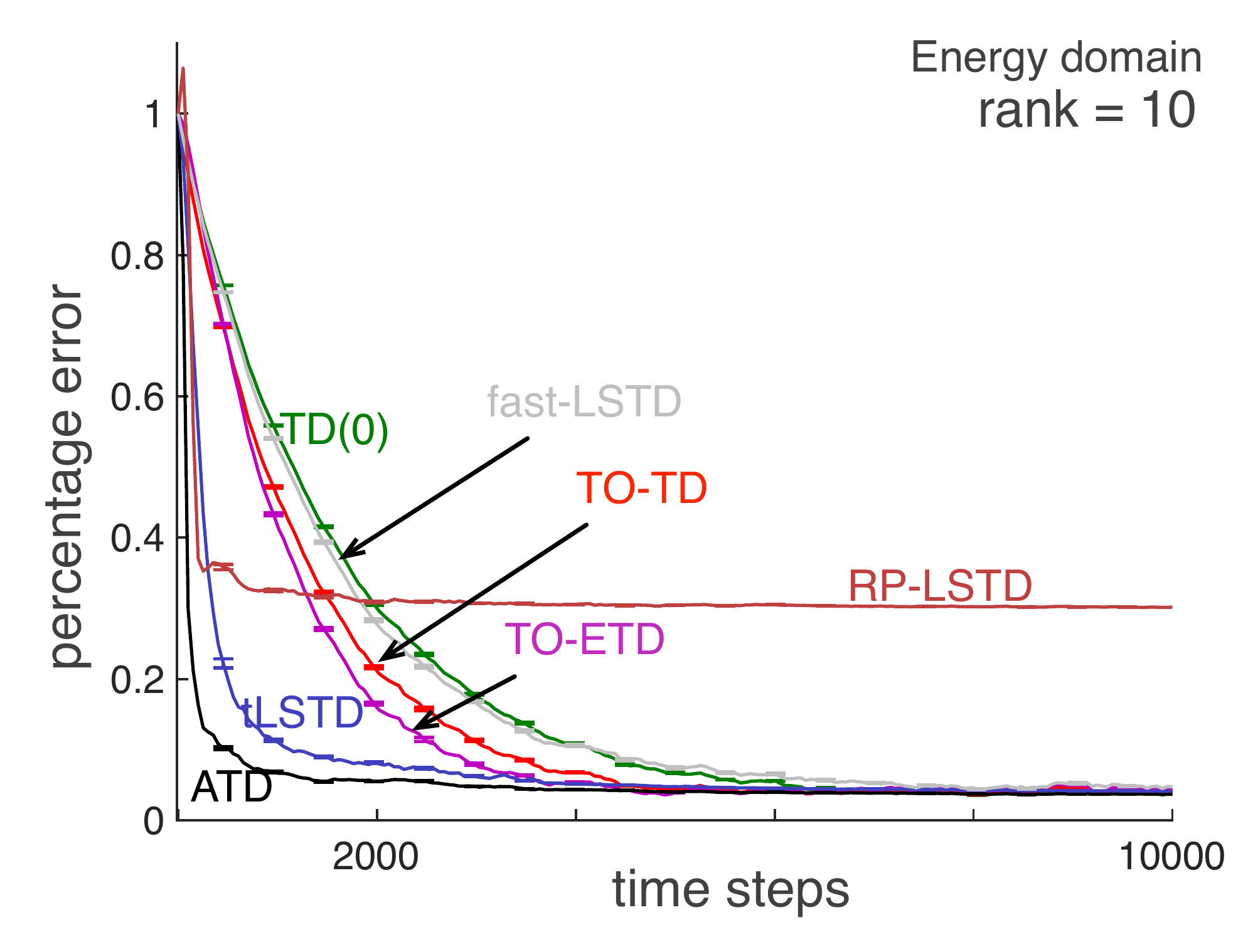}
     \caption{\label{energy} Learning curves on energy allocation domain with rank equal to 10. Here we see the clear difference of the effect of rank on these two methods. ATD is only using the curvature information in $\Arank$, to speed learning, whereas tLSTD uses $\Arank$ in a closed form solution. 
     }
\end{figure}

Due to the size of the feature vector we excluded LSTD from the results, while iLSTD was also excluded due to it's slow runtime and poor performance in Mountain Car. Note that though iLSTD avoids $O(\xdim^2)$ computation per step for sparse features, it still needs to store and update an $O(\xdim^2)$ matrix, and so does not scale as well as the other sub-quadratic methods.

\end{document}